%% file: main.tex
\documentclass[11pt]{article} 
\usepackage[left=1in,top=1in,right=1in,bottom=1in,letterpaper]{geometry}

\usepackage{appendix}
\usepackage{color}
\usepackage{url}
\usepackage{graphicx}
\usepackage{multirow}
\usepackage{amssymb}
\usepackage{amsthm}
\usepackage{amsmath}
\usepackage[inline]{enumitem}
\usepackage{hyperref}
\usepackage{microtype}      
\usepackage{graphicx}

\usepackage{tcolorbox}
\usepackage{listings}
\usepackage{xcolor}
\usepackage{wrapfig}
\usepackage{cleveref}
\usepackage{makecell}
\usepackage{subcaption} 
\usepackage{multirow}
\usepackage{multicol}
\usepackage{blindtext}
\usepackage{enumitem}
\usepackage{calc}
\usepackage{ifthen}
\usepackage{bbding}
\usepackage{amsmath,amsfonts, amsthm, amssymb}
\usepackage{color,graphicx,overpic}
\usepackage{hyperref}
\usepackage{bbm}
\usepackage[round, sort&compress]{natbib}
\usepackage{subcaption}
\input{Command}

\title{When does Chain-of-Thought Help: A Markovian Perspective}

\author{
Zihan Wang
\ \ \ \ \ \  
Yijun Dong
\ \ \ \ \ \  
Qi Lei
\vspace{0.3cm}
\\
New York University
}
\date{}

\begin{document}

\maketitle
\vspace{-0.5cm}
\begin{abstract}
  Chain-of-Thought (CoT) prompting is a widely used inference-time technique for improving reasoning, yet its gains are uneven across tasks. We analyze when and why CoT helps by modeling the step-wise reasoning trajectory as a Markov chain. Each intermediate step is a state and the dependence between steps is captured by a transition kernel. Our theory identifies transition alignment, whether instances share a common step-wise transition kernel, as the key determinant of CoT’s effectiveness. When transitions are identical across steps, CoT reduces inference-time sample complexity: fewer context sample trajectories suffice to recover the final decision. In contrast, when transitions differ across steps, these gains can vanish. We further quantify how noise in intermediate steps modulates CoT’s benefit. Beyond theory, we design synthetic benchmarks that isolate these factors to complement prior results on real-world tasks and to empirically validate our predictions.
\end{abstract}

\section{Introduction}

Chain-of-Thought (CoT) prompting has become a de facto approach of inference-time scaling for multi-step reasoning with large language models (LLMs). 
It demonstrates substantial improvements on math and symbolic tasks \citep{wei2022chain,kojima2022zeroshot,wang2023self,zhou2022least,yao2023tree} but modest or mixed effects on others \citep{sprague2025cot}. Noisy or unfaithful intermediate reasoning steps have even been shown to mislead predictions of CoT and bring worse performance than direct inference, despite the additional computation of CoT \citep{prabhakar2024deciphering}. 

Since its first introduction, CoT has been extensively investigated empirically and theoretically.
Mechanism-centric work explores how to aggregate inference-time trajectories (e.g., self-consistency, tree search) and offers qualitative rationales \citep{wang2023self,yao2023tree}. 
Task-centric works analyzed catalogs where CoT tends to help (e.g., math, symbolic, multi-hop) and where it does not \citep{sprague2025cot,cobbe2021gsm8k,lewkowycz2022minerva,yang2018hotpotqa}.
Along this task-centric perspective, a critical missing piece is a rigorous yet intuitive theoretical model of CoT that explains its successes and failures on different downstream tasks.
This motivates our central research questions:
\begin{center}
    \emph{When does CoT provably outperform direct inference? \\ Can we distinguish beneficial cases of CoT from its failures via measurable structural properties of the downstream task?}
\end{center}
Mechanistic understandings of these questions will (i) provide first-principled guidelines for where to apply CoT and how to structure in-context demonstrations, (ii) clarify which aspects of the downstream tasks make CoT appear strong or fragile, and (iii) steer the design of clean evaluation metrics for CoT that disentangle different sources of failures.
This work makes a step toward filling these gaps through theory-backed accounts and controlled experiments.

In light of the stochastic and autoregressive nature of CoT, along with the finite context windows in practice, we adapt the Markovian view of reasoning from \cite{xie2021iclbayes,prystawski2023think,besta2024graph,abbe2024far,sanford2024understanding,kim2025metastable} and compare CoT against direct inference through a unified lens.
Conceptually, we formalize step-wise reasoning as a trajectory over latent states and study how two ingredients govern sample efficiency: \textbf{transition alignment} across steps (``same skill'' vs.\ ``different skills'') and \textbf{noise} separating correct from competing options. Intuitively, identical local rules can enhance per-step votes through the reasoning process, leading to a performance gain for CoT. And because composed, end-to-end margins contract under uncertainty, the scale of intermediate-step noise also plays an important role in reasoning. We keep the formal modeling and decision rule in Section~\ref{sec:markov}.

Specifically, we provide a compact analysis of inference-time CoT that isolates the two factors above and turns them into concrete predictions. Our theory shows when CoT enjoys a structural $1/T$-type improvement and when it does not, and also explains how noise affects in Section~\ref{sec:theory}. The proof sketch of main results is in Section~\ref{sec:proof sketch}. In Section~\ref{sec:exp}, we then validate these predictions in clean synthetic settings that manipulate alignment and noise directly, and in a small arithmetic task that offers a practical sanity check.

Our contributions are as follows, which offer a compact answer to our essential question together.
\begin{enumerate} 
    \item \textbf{A Markovian modeling and decision rule for inference-time aggregation.} We model reasoning as a Markov chain over a finite state space and analyze inference time only. This isolates how context samples translate into decisions, aligning with trajectory aggregation practices such as self-consistency and search \citep{wang2023self,yao2023tree}. \item \textbf{Theory that pinpoints the factors behind CoT's benefits.} We derive sample-complexity bounds for (a) direct inference, and (b) CoT with aligned and misaligned transitions. From the theoretical results, we demonstrate two key factors in CoT performance.
    \item 
    \textbf{Clean synthetic tests that manipulate alignment and noise.} We design clean benchmarks showing that CoT reduces the sample budget most when transitions align, and its relative advantage increases with noise. We also verify our results on a more realistic but still structured modular addition task.
\end{enumerate} 


\section{Related Works}
\paragraph{When does CoT help, (and when it doesn't)?}
Large‑scale empirical assessments unveil that the largest gains of CoT tend to concentrate on math/symbolic tasks, with much smaller effects elsewhere; failures often arise from noisy or unfaithful intermediate steps, suggesting the need for aggregation or verification to counteract error propagation~\citep{prabhakar2024deciphering,sprague2025cot,zheng2025curse}. 
These observations underscore that task structure and intermediate noise govern the returns to inference‑time compute, and thereby motivate our theoretical investigation of skill diversity and intermediate-step noise. 

\paragraph{Inference-time scaling via reasoning trajectory aggregation.}
Inference in autoregressive models consists of forward passes through the trained network to generate tokens, a process once regarded as lightweight and essentially fixed in cost. With the unprecedented capacity of modern LLMs, however, reasoning has emerged as an effective attribution of computation, scaling which significantly improves inference-time performance~\citep{jaech2024openai,guo2025deepseek}.
CoT prompting~\citep{wei2022chain} and its zero-shot variant~\citep{kojima2022zeroshot} established the basic paradigms of eliciting intermediate reasoning steps, which can be scaled via assorted schemes.
\begin{enumerate*}[label=(\roman*)]
    \item Subsequent methods that explore and aggregate multiple reasoning trajectories~\citep{zhou2022least,chen2022pot,wang2023self,yao2023tree,snell2024scaling}, like self-consistency and best-of-n, have been shown to amplify these gains even further. 
    \item In addition, CoT has been enhanced through various reinforcement-learning-based methods, such as Monte-Carlo Tree Search~\citep{silver2018general,trinh2024solving,feng2023alphazero,xie2024monte} and process reward models~\citep{lightman2023let,uesato2022solving}.
    \item Alternatively, self-improvement that refines reasoning trajectories using the model's own judgments~\citep{zelikman2022star,hosseini2024v,kumar2024training} has also been used to boost performance at inference.
\end{enumerate*}
Collectively, these empirical successes of reasoning as inference-time scaling motivate our focus on the structure of multi‑step reasoning trajectories and inference‑time sample complexity.

\begin{figure*}[t]
  \centering

  \begin{subfigure}[t]{0.7\textwidth}
    \centering
    \includegraphics[width=\linewidth]{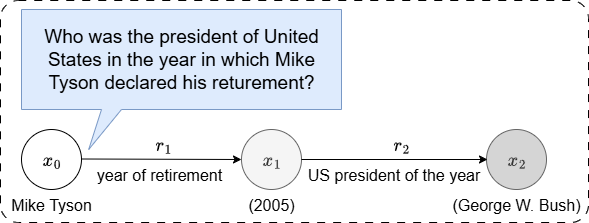}
    \caption{}
    \label{fig:modeling_example}
  \end{subfigure}
  \hfill
  \begin{subfigure}[t]{0.27\textwidth}
    \centering
    \includegraphics[width=\linewidth]{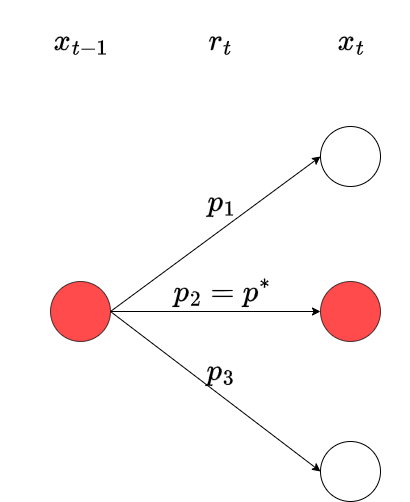}
    \caption{}
    \label{fig:markov_choice}
  \end{subfigure}

  \caption{(a) \textbf{Example of Markov-style decomposition.}
      An illustrative instance, where the original text is selected from \citep{mavi2024multi}, is parsed into an initial state $x_0$ and a sequence of relations/operators $r_{1:T}$. The correct intermediate steps $x_{1:T}$ are also included.
      The example highlights how text data are collapsed into a formalized sequence, making the intermediate operations explicit for CoT demonstrations. (b) \textbf{Stepwise Markov dynamics.}
      Each $r_t$ induces a transition kernel $P^{(t)}$ from $x_{t-1}$ to $x_t$ and the next node is generated randomly following the transition.
      At inference, the ground truth intermediate steps are the nodes with the largest probability $p^*$ for each step.}
  \label{fig:markov_overview}
\end{figure*}

\paragraph{Graph and Markovian views on LLMs and CoT.}
Reasoning on graphs, either externally through knowledge graphs or implicit via graph algorithms, provides an abstraction for multi-step tasks~\citep{xie2021iclbayes,prystawski2023think,besta2024graph,abbe2024far,sanford2024understanding,kim2025metastable}.
Several works interpret transformers or their decoding dynamics through a Markov lens, including token-level connections between autoregressive LLMs and finite-state Markov chains~\citep{zekri2024large}, the relation between self-attention and Markov models~\citep{ildiz2024self}, and mechanistic studies of in-context learning through induction heads~\citep{edelman2024evolution,makkuva2024attention}. 
In parallel, several recent approaches operationalize an explicit Markovian state abstraction for long reasoning: compress long CoT histories into short states for mathematical reasoning \citep{yang-etal-2025-markov}, enforce constant-size states via
chunked reasoning with carryover \citep{aghajohari2025markovianthinker}, and train a text-based bottleneck that
makes the intermediate CoT more causally load-bearing \citep{viteri2024markoviantransformers}.
Closer to our setting, \cite{kim2025metastable} casted CoT as a metastable Markov process on thought space, explaining why long reasoning trajectories exhibit phase transitions and suggesting search/RL/distillation schemes to escape local basins.
We differ in what is modeled and which question is answered: by positing a task-level Markov chain over latent reasoning states, we analyze the inference-time sample complexity of CoT under different structures of the Markov chain corresponding to different skill diversity and intermediate-step noise of the downstream task.

\paragraph{More theoretical accounts of CoT.}
Beyond the graph and Markovian views, the benefits and mechanisms of CoT have been extensively investigated from various theoretical perspectives, including expressivity~\citep{feng2023towards,chiang2023tighter,merrill2024expressive,li2024chain}, 
lower bounds in step complexity of CoT \citep{amiri2025lowerbounds, barcelo2025eh_rank}, simulation of optimization algorithms~\cite{huang2025transformers}, inference-time statistical efficiency~\citep{altabaa2025cot,joshi2025theory,wen2024sparse,huang2025sample,kim2025transformers,hu2024unveiling}, and the emergence of CoT/reasoning capability from pretraining~\citep{nichani2024transformers,li2024nonlinear,yang2025multi,wang2025learning}.
Our Markovian perspective of CoT is orthogonal and complementary to these, explaining not only the empirical successes but also failures of CoT in terms of sample complexity at the inference time.

\section{Markovian Modeling}
\label{sec:markov}

We model an instance as a sequence of $T$ relations (local rules/operators) applied to an initial state. The ground truth of an instance contains an input $x_0$, a final output $x_T$, relations $(r_1,\dots,r_T)$ and latent intermediate steps $(x_1,\dots,x_{T-1})$. The model only observes the input and relations $(x_0, r_1, \dots, r_T)$ and must infer $x_T$. With CoT, the model may also emit predictions $\hat{x}_1, \dots, \hat{x}_{T-1}$ to estimate intermediate steps sequentially conditioned on the growing prefix; without CoT it directly predicts $\hat{x}_T$. An illustrative example is shown in Fig.~\ref{fig:modeling_example}.

This abstraction treats the textual problem as a latent state transformation pipeline: each $r_t$ induces a transition kernel $P^{(t)}$ that maps the current state $x_{t-1}$ to a distribution over next states $x_t$. Many reasoning tasks naturally decompose this way, for example carrying updates in arithmetic \citep{cobbe2021gsm8k,lewkowycz2022minerva,chen2022pot}), symbolic rewrites \citep{clark2020rulertaker,tafjord2021proofwriter}), knowledge-graph paths \citep{yang2017neurallp,hamilton2018embedding,ren2020betae}, and composing path in multi-hop QA \citep{yang2018hotpotqa,khot2020qasc,trivedi2022musique}).
By collapsing text into a finite set of task-relevant states, we separate \emph{form} (how a chain is written) from \emph{mechanism} (what operation is performed), which is exactly what CoT aims to expose: a trajectory of intermediate operations rather than a single terminal answer.

We work with a finite state space $[k]$ as a compact abstraction of task-relevant equivalence classes (e.g., arithmetic states, intermediate answers in multi-hop QA), not a literal vocabulary. In many benchmarks, a small latent alphabet suffices to determine the final decision. The end-to-end kernel is
\begin{equation}
    \label{eq:P-Q}
Q := P^{(1)} P^{(2)} \cdots P^{(T)},
\end{equation}

so the $i$-th row $Q_i$ specifies the distribution of $x_T$ given $x_0 = i$, and the predictor outputs the one-hot index of the maximal entry of $Q_i$.

Our analysis focuses purely on \emph{inference time}. The model receives $n$ context samples that share the same relation tuple $(r_1,\dots,r_T)$ and are drawn with $x_0 \sim \mu$ and transitions governed by $\{P^{(t)}\}$. A direct-inference sample reveals only $x_T$; a CoT sample reveals the full path $(x_0,\dots,x_T)$. The modeling choice aligns with recent formalisms that cast decoding and trajectory aggregation as Markov-like dynamics \citep{kim2025metastable} and with algorithmic views of in-context learning \citep{xie2021iclbayes,akyurek2022what,zhang2024trainedICL}.

A key structural distinction is whether $P^{(1)}=\cdots=P^{(T)}$ (homogeneous/aligned) or the kernels differ by step (heterogeneous/misaligned). In the first case, a single trajectory provides $T$ glimpses of the same local rule, allowing per-step votes to be pooled. In the second, per-step observations do not reinforce any single kernel, so pooling does not have the same effect. This is the precise mathematical counterpart of “same skill versus different skills across steps.”

\section{Theoretical Analysis}
\label{sec:theory}
Our goal in this section is to isolate \emph{which factors} govern the benefit of inference-time CoT theoretically. To keep the analysis transparent, we adopt the same decision rule for \textbf{both} direct inference and CoT: the model behaves like a simple \emph{counter} that uses the context samples to estimate class frequencies and then predicts the \emph{hard-max} index (argmax) for the relevant row of the transition kernel (Fig.~\ref{fig:markov_choice}).
Within this lens, two ingredients determine whether CoT helps: (i) \textbf{transition alignment}, and (ii) \textbf{noise}. The first drives the structural $1/T$-type advantage and the second explains when noise amplifies CoT’s relative gains.

\subsection{Direct Inference}

In direct inference, each context sample reveals only the end-to-end output $x_T$ under kernel $Q$, defined in Eq.~\eqref{eq:P-Q}. With the shared ``count-and-argmax" rule, recovering the argmax of each row $Q_i$ reduces to multi-class frequency estimation conditioned on $x_0=i$. Reliable identification requires that every input state appears with nontrivial frequency and that the optimal class has a positive gap.

\begin{assum}[Coverage of distribution]
\label{assum:init dist}
    The multinomial distribution $\mu $ satisfies that the smallest element of $\mu$ is strictly greater than 0 and denoted by $\mu_{\min}$.
\end{assum}

\begin{assum}[Positive margin]
\label{assum:Q margin}
    The ground truth transition matrix $Q$ satisfies that 
    $$Q_{i,j^*(i)}-Q_{ij}\ge\Delta_Q>0$$
    for all $i$ and $j\neq j^*(i)$, where $j^*(i)=\arg\max_{i}Q_{ij}$.
\end{assum}

\noindent Under these conditions, the usual frequency-estimation rate appears:

\begin{thm}[Direct inference]
\label{thm:direct}
    Under Assumption \ref{assum:init dist} and \ref{assum:Q margin}, if the number of samples $n$ satisfies
    $$n= \Theta\left(\frac{\log(k/\delta)}{\mu_{\min}\Delta_Q^2}\right),$$
    then with probability at least $1-\delta$, we have $\hat{j}^*(i)=\arg\max_j Q_{ij}$ for all $i$ with direct inference.
\end{thm}

Two observations already connect to our main question. First, the dependence on $\mu_{\min}$ simply reflects coverage: under the counting rule, you cannot learn the argmax for states you rarely observe. Second, the end-to-end margin $\Delta_Q$ may be much \emph{smaller} than per-step margins because local uncertainties compound through $L$ transitions. This foreshadows why CoT—which counts \emph{locally}—can be comparatively robust when transitions align.

\subsection{Chain of Thought}

With CoT, each context sample exposes the entire path $(x_0,x_1,\ldots,x_T)$, so the same counting rule is applied \emph{per step}: estimate the hard-max index of each row of the relevant transition kernel from the observed transitions and then compose these local argmax decisions. To link local and global decisions we assume a standard greedy-to-global consistency:

\begin{assum}[Local-global consistency]
\label{assum:consistent max}
    We assume for all $t$ and $i$, each row $P^{(t)}_i$ of the transition matrix has a single max index and denote this maximum index by $I_t(i)$. We additionally assume that $$I_T\circ\cdots\circ I_1(x_0)=\arg\max_{j}Q_{x_0,j}.$$
\end{assum}

\noindent The impact of CoT then hinges on \textbf{how the per-step kernels relate}. We make this explicit by contrasting the homogeneous and heterogeneous regimes.

\paragraph{Homogeneous (aligned) transitions.} When $P^{(1)}=\cdots=P^{(T)}=:P$, every trajectory provides $T$ observations of the \emph{same} kernel. Under the counting rule, this multiplies the effective number of local votes per trajectory, subject to standard mixing/coverage and a per-step margin:

\begin{assum}
    \label{assum:transition}
    The pseudo-spectral gap of $P$ is $\gamma>0$. The stationary distribution of $P$ is $\pi$ and the minimal element of $\pi$ is positive $\pi_{\min}>0$. The input distribution $\mu$ satisfies $\sqrt{\chi^2(\mu||\pi)}=:\chi_0$. The ground truth transition matrix $P$ satisfies
    $$P_{i,j^*}-P_{ij}\ge \Delta_P>0$$
    for all $i$ and $j\neq j^*(i)$, where $j^*(i)=\arg\max_i{P_{ij}}$.
\end{assum}

\begin{thm}[Homogeneous CoT]
\label{thm:same}
    Under Assumptions \ref{assum:init dist}, \ref{assum:consistent max} and \ref{assum:transition}, if the number of samples $n$ satisfies
    $$n=\Theta\left(\left(\frac{1}{T\pi_{\min}\Delta_P^2r}+\frac{1}{T\pi^2_{\min}r^2}\right)\log(k/\delta)\right),$$
    where $r=1-\frac{\chi_0}{T\pi_{\min}\gamma}$, then with probability at least $1-\delta$, we have $\hat{j}^*(i)=\arg\max_{j}Q_{ij}$ for all $i$ with CoT inference with homogeneous transitions.
\end{thm}

\begin{figure*}[t]
  \centering
   \begin{subfigure}[t]{0.48\textwidth}
  \centering
    \includegraphics[width=\linewidth]{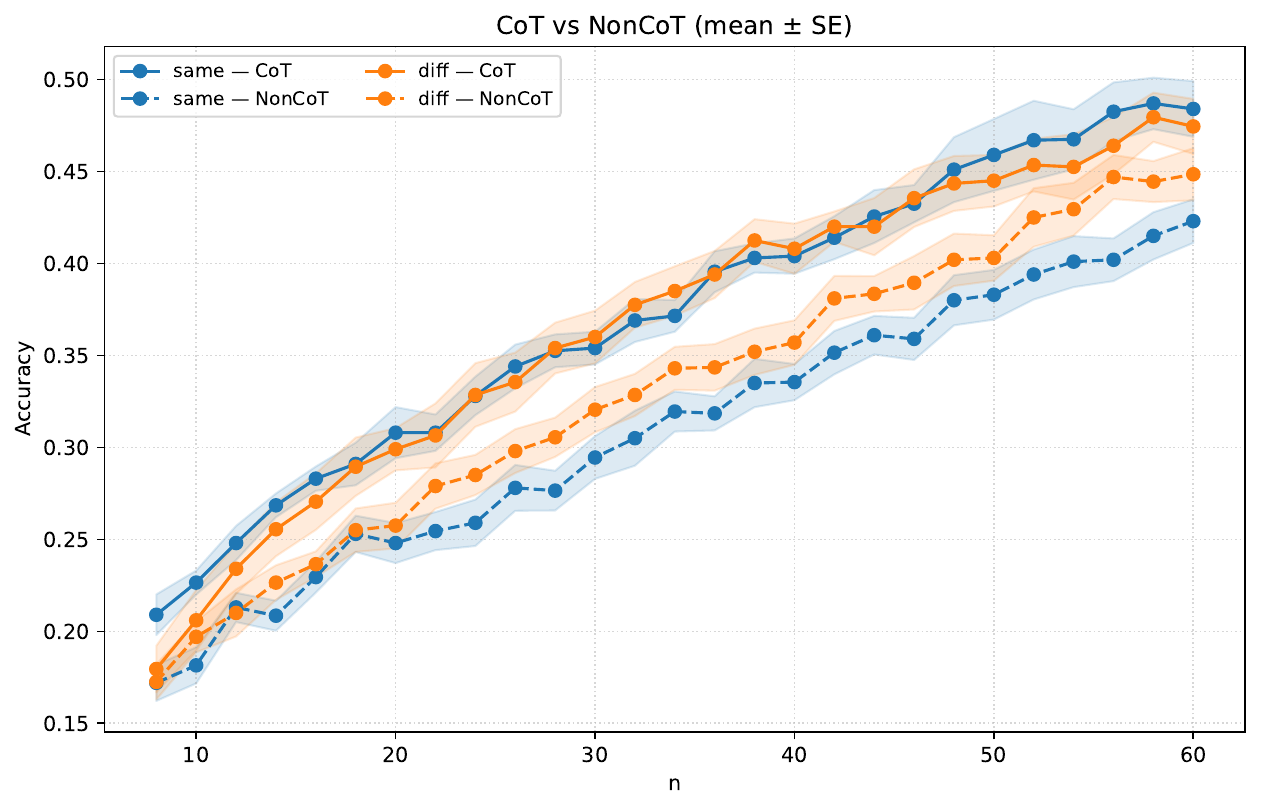}
    \end{subfigure}
  \hfill
   \begin{subfigure}[t]{0.48\textwidth}
  \centering
    \includegraphics[width=\linewidth]{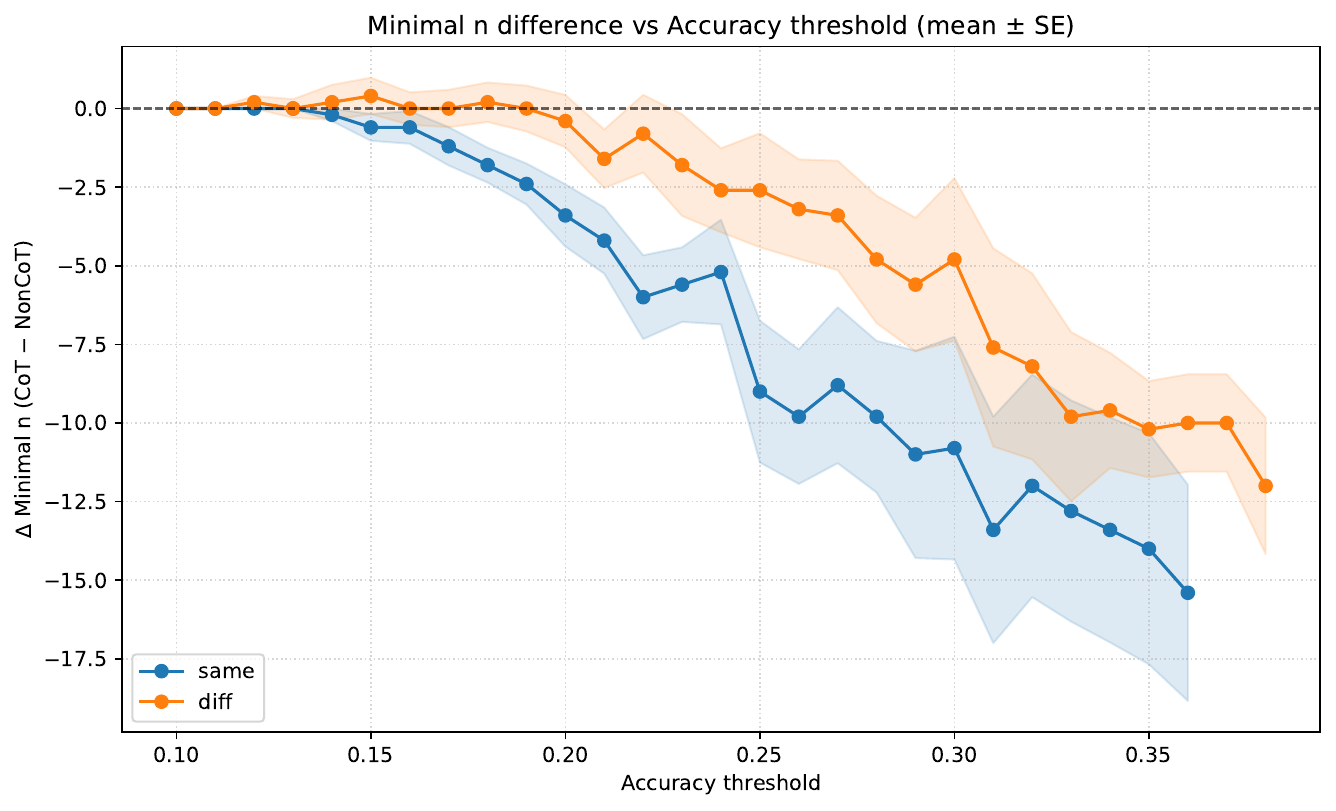}
    \end{subfigure}
  \caption{\textbf{Synthetic alignment study (\textbf{same} vs.\ \textbf{diff}).} \emph{Left:} Accuracy as a function of the number of context samples $n$ comparing CoT (solid) and NonCoT (dashed). CoT yields a consistently larger improvement under \textbf{same} than under \textbf{diff}. \emph{Right:} Sample-complexity proxy $\Delta n(\tau)$ (CoT$-$NonCoT) vs.\ target accuracy $\tau$. More negative values indicate fewer samples required by CoT. The widening gap in the \textbf{same} condition is consistent with the theoretical $1/T$-type gain for aligned kernels. Shaded regions denote mean~$\pm$~SE.}
  \label{fig:synthetic_same_diff}
\end{figure*}

Three messages tie back to our factorization. (A) \textbf{Alignment} produces the structural $1/T$ gain: the same counting rule now accrues $T$ local votes per trajectory for a single kernel. Note that if $T$ is large enough, then the sample complexity of homogeneous CoT will always be smaller than direct inference, showing alignment is the most important factor here. (B) The \textbf{margin that matters} is the local $\Delta_P$, typically larger than the composed $\Delta_Q$; hence CoT’s relative robustness to compounding uncertainty. (C) The term $r$ quantifies how fast trajectories diversify across rows (mixing/coverage), ensuring those local counts are informative rather than redundant.

\paragraph{Heterogeneous (misaligned) transitions.} When the kernels vary by step, a trajectory’s $T$ local observations are \emph{split} across different kernels. The same counting rule still applies, but the local estimates no longer pool toward a single transition. Identification thus has to hold at every time, controlled by stepwise coverage and a uniform local margin:

\begin{assum}[Heterogeneous CoT]
    \label{assum:diff}
    The intermediate distributions satisfy there exists $q_{\min}>0$ such that
    $$q_{\min}\le \min_{0\le t\le T-1}\min_{i}\mu^{(t)}_i.$$
    The ground truth transition matrices satisfy
    $$P^{(t)}_{i,j^*}-P^{(t)}_{ij}\ge \Delta>0$$
    for all $t$, $i$ and $j\neq j^*_t(i)$, where $j_t^*(i)=\arg\max_i{P^{(t)}_{ij}}$.
\end{assum}

\begin{thm}
    \label{thm:diff}
    Under Assumptions \ref{assum:init dist}, \ref{assum:consistent max} and \ref{assum:diff}, if the number of samples $n$ satisfies
    $$n=\Theta\left(\frac{\log(Tk/\delta)}{q_{\min}\Delta^2}\right),$$
    where $r=1-\frac{\chi_0}{T\pi_{\min}\gamma}$, then with probability at least $1-\delta$, we have $\hat{j}^*(i)=\arg\max_{j}Q_{ij}$ for all $i$ with CoT inference with heterogeneous transitions.
\end{thm}

Here there is no guaranteed $1/T$ speedup. The same trajectory yields evidence about \emph{different} kernels, so the counting rule cannot pool those votes as in the aligned case, and a $\log T$ term appears instead.
\begin{remark}Comparisons between Sample Complexity of Heterogeneous CoT and Direct Inference.
    \begin{enumerate}
        \item $q_{\min}\le \min_{i}\mu_i^{(0)}=\mu_{\min}$, showing that the coverage in misaligned CoT is worse than direct inference.
        \item Both cases $\Delta_Q>\Delta$ and $\Delta_Q<\Delta$ exist. For the former case one can take $$P^{(1)}=\begin{pmatrix}
            0.6 & 0.4 \\
            0.1 & 0.9
        \end{pmatrix}\ \textrm{and} \ P^{(2)}=\begin{pmatrix}
            0.9 & 0.1 \\
            0.3 & 0.7
        \end{pmatrix},$$ and for the latter case $$P^{(1)}=\begin{pmatrix}
            0.8 & 0.2 \\
            0.8 & 0.2
        \end{pmatrix}\ \textrm{and}\ P^{(2)}=\begin{pmatrix}
            0.7 & 0.3 \\
            0.3 & 0.7
        \end{pmatrix}.$$
    \end{enumerate}
\end{remark}

This remark shows that the sample complexity of heterogeneous CoT can be smaller than direct inferece, due to the $\log T$ factor, worse coverage and possibly smaller margin.
Nonetheless, heterogeneous CoT can still be competitive when the per-step margin $\Delta$ is substantially larger than the composed margin $\Delta_Q$, which is the case in our experiments.

The conclusions of this section are as follows.
\begin{itemize}
\item Under a common, simple \emph{count-and-argmax} decision rule, the factor that controls structural gains is \textbf{transition alignment}: identical per-step kernels yield a $1/T$-type reduction in sample complexity; different kernels do not.
\item The factor that controls sensitivity is the \textbf{noise (margin)}: CoT depends on the local margin $\Delta_P$, while direct inference depends on the composed margin $\Delta_Q$. Because $\Delta_Q$ typically contracts faster under noise, CoT’s \emph{relative} advantage grows as intermediate-step noise increases.
\end{itemize}



\section{Proof Sketch}
\label{sec:proof sketch}
The bounds in Section~\ref{sec:theory} reduce hard-max prediction in our Markovian model to a sequence of elementary \emph{top-1 identification} problems for multinomial distributions. Intuitively, under the count-and-argmax rule, we only need (i) enough coverage of the relevant rows to estimate their class frequencies, and (ii) a positive margin so the empirical argmax matches the true argmax. Direct inference and homogeneous CoT each require identifying the top class for one $k\times k$ transition matrix (either $Q$ or $P$), whereas heterogeneous CoT requires doing so for all $T$ per-step matrices. This is why a uniform top-1 identification lemma for multinomials is the core primitive behind all three theorems.

\begin{lem}[Chapter 2, \citep{BoucheronLugosiMassart2013}]
\label{lem:multinomial-top1}
For a multinomial distribution $p$, we write $p_{(1)}\ge p_{(2)}\ge\cdots$ for its ordered elements. We assume the margin $\Delta_p:=p_{(1)}-p_{(2)}>0$. If 
there exist absolute constants $C>0$ such that 
$$N\ge \frac{C}{\Delta^2}\,\log\!\frac{k}{\delta},$$
then the empirical argmax equals to the argmax with probability at least $1-\delta$.
\end{lem}

The lemma converts a margin lower bound into the familiar $O(\Delta^{-2}\log(k/\delta))$ sample requirement, so the main task is to lower bound the effective counts with which each row is observed. Let $N_i^{(0)}$ be the number of context samples whose initial state is $x_0=i$, and $N_i^{(t)}$ the number of visits to state $i$ at step $t$. For direct inference, it suffices to control $\min_i N_i^{(0)}$ because we only estimate the row of $Q$ indexed by $x_0$. A Chernoff bound (e.g., \citealp{Hoeffding1963}) gives $\min_i N_i^{(0)} \gtrsim c\,N\,\mu_{\min}$ with high probability for an absolute constant $c>0$, which yields Theorem~\ref{thm:direct}.

For heterogeneous (misaligned) CoT, we must identify the top class in each per-step kernel $P^{(t)}$, so we need a uniform lower bound on the step-wise counts $\min_{t,i} N_i^{(t)}$. Since the states at step $t$ are distributed as $\mu^{(t)}=\mu P^{(1)}\cdots P^{(t)}$, we obtain $\min_i N_i^{(t)} \gtrsim c\,N\,\min_i \mu^{(t)}_i$; invoking the assumption $q_{\min}\le \min_{t,i}\mu^{(t)}_i$ and a union bound over $t\in\{0,\dots,T-1\}$ controls the simultaneous deviations across steps. Plugging this coverage into the multinomial lemma with the local margin $\Delta$ gives the heterogeneous CoT bound (Theorem~\ref{thm:diff}), including the logarithmic dependence on $T$ from the union bound.

The homogeneous (aligned) CoT case is more delicate and also where the structural gain appears. Because each trajectory contributes all $T$ transitions of the same kernel $P$, the total number of usable local observations scales like $NT$. Define the aggregated counts $N_i:=\sum_{t=0}^{T-1} N_i^{(t)}$. If the initial distribution $\mu$ were already stationary for $P$, then each step would sample from $\pi$ and we would have $N_i\approx NT\,\pi_i$ up to concentration, immediately yielding a $1/T$ improvement when the lemma is applied with margin $\Delta_P$. In general $\mu$ need not equal $\pi$, so we quantify the bias and dependence via the $\chi^2$-divergence $\chi_0$ between $\mu$ and $\pi$ and the pseudo-spectral gap $\gamma$ of $P$. These two quantities govern how quickly the chain forgets its start and how many “effectively independent’’ transitions we get along a path. The next bound formalizes this intuition.

\begin{lem}
\label{lem:min-count-markov}
    There exist absolute constants $c_1,c_2>0$ such that, for any $\delta\in(0,1)$, with probability at least $1-\delta$,
\[
\min_{i} N_i
\ge
NT\,\pi_{\min}
-
\frac{N\,\chi_0}{\gamma}
-
\sqrt{\frac{c_1\,NT}{\gamma}\,\log\!\frac{k}{\delta}}
\;-\; c_2,
\]
where $\pi_{\min}:=\min_i \pi_i$. 
\end{lem}

\begin{remark}
    \begin{enumerate}
        \item The bias term $N\chi_0/\gamma$ comes from the the error between $\mu P^t$ and $\pi$: $\frac{1}{T}\sum_{t=0}^{T-1}\|\mu P^t-\pi\|_{\mathrm{TV}}\lesssim \chi_0/(T\gamma)$ \citep{paulin2015psg}.
        \item 
The fluctuation term derives from Bernstein–type concentration for bounded functions along Markov chains with pseudo–spectral gap \citep{paulin2015psg}.
    \end{enumerate}
\end{remark}

\begin{figure*}[t]
  \centering
  \begin{subfigure}[t]{0.48\textwidth}
  \centering
    \includegraphics[width=\linewidth]{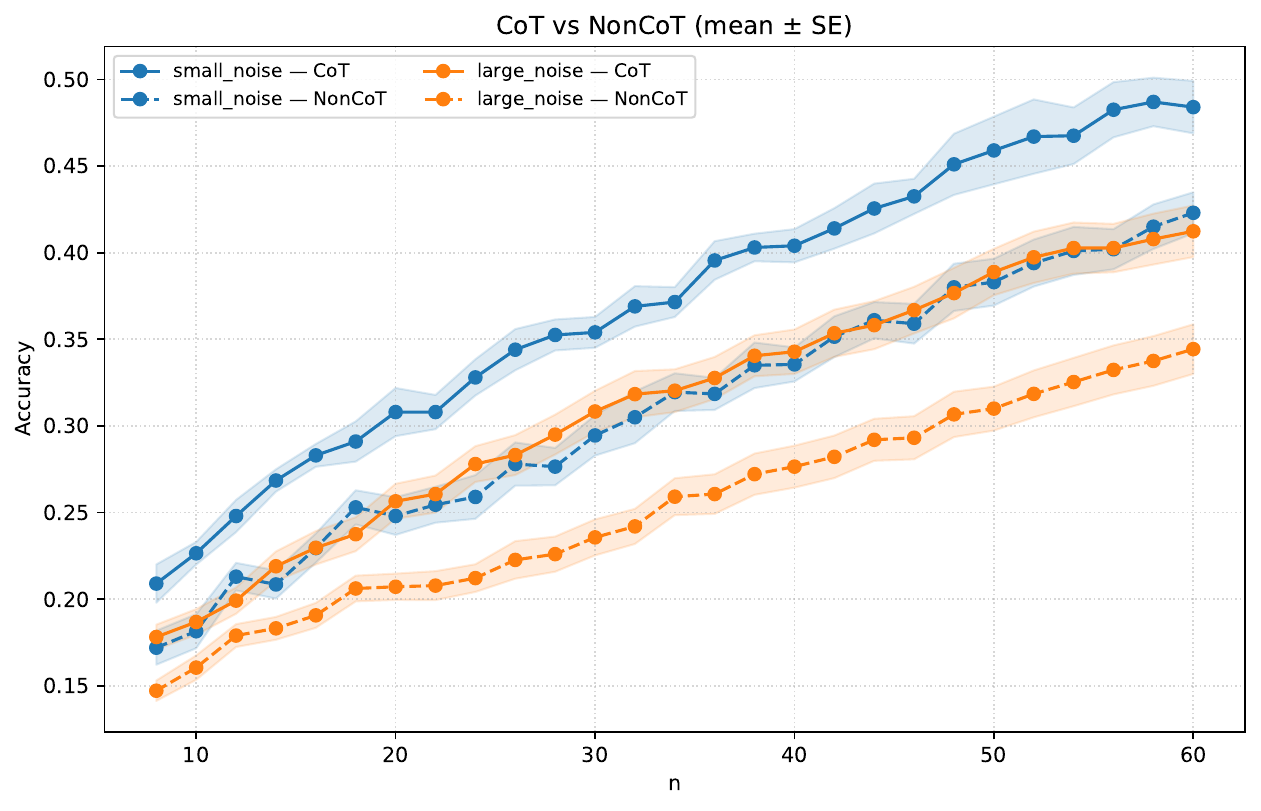}
    \end{subfigure}
  \hfill
  \ \begin{subfigure}[t]{0.48\textwidth}
  \centering
    \includegraphics[width=\linewidth]{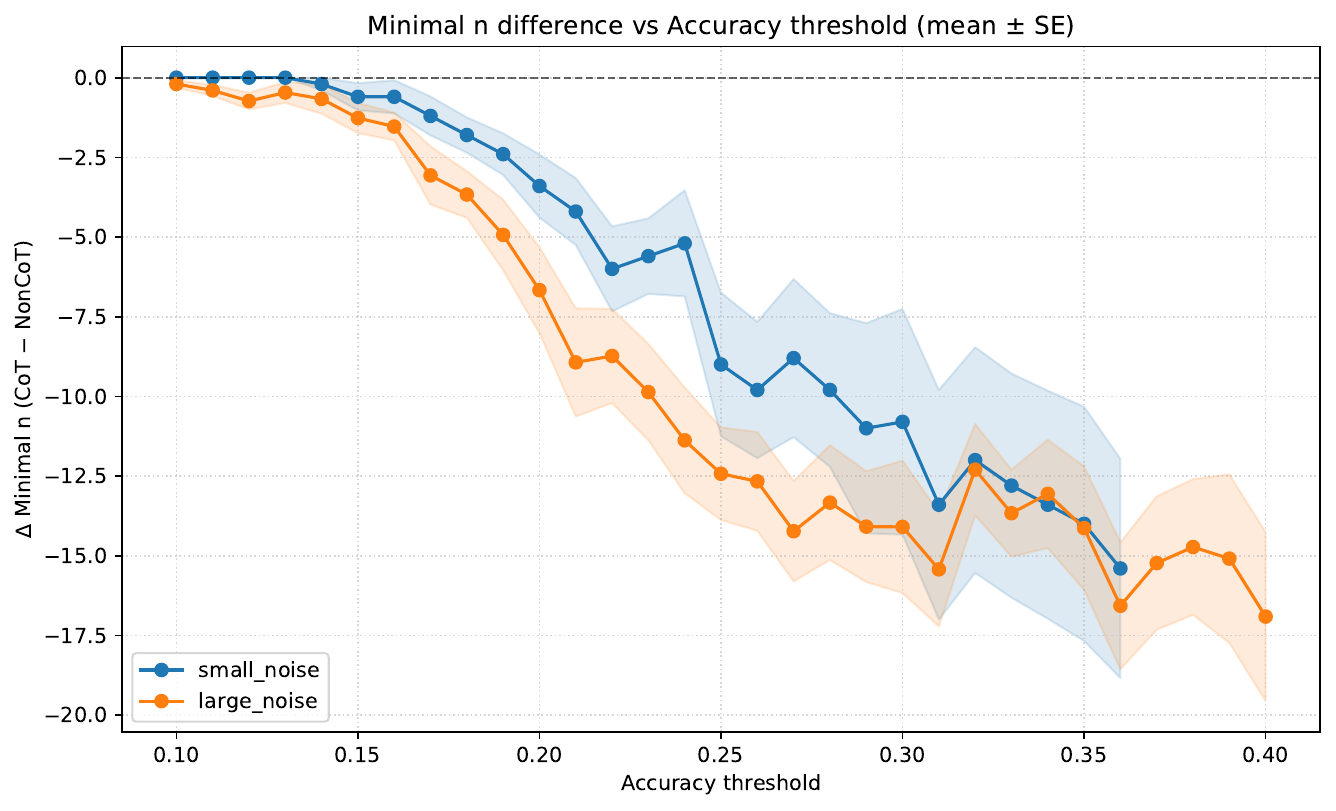}
    \end{subfigure}
  \caption{\textbf{Noise ablation under aligned transitions.} \emph{Left:} Accuracy vs.\ $n$ for CoT/NonCoT under \textbf{small} and \textbf{large} intermediate-step noise. \emph{Right:} $\Delta n(\tau)$ vs.\ $\tau$. CoT reduces the required context budget more when noise is higher, reflecting that the global margin $\Delta_Q$ shrinks faster than the local margin $\Delta_P$ as stochasticity increases. Shaded regions denote mean~$\pm$~SE.}
  \label{fig:synthetic_noise}
\end{figure*}

Combining this coverage bound with the multinomial lemma yields the homogeneous CoT rate in Theorem~\ref{thm:same}: the leading $1/T$ term reflects $T$ informative transitions per trajectory for the same kernel, and the factor $r=1-\chi_0/(T\pi_{\min}\gamma)$ captures the erosion due to initialization bias and mixing.


\section{Experiments}
\label{sec:exp}
The empirical benefits of chain-of-thought (CoT) prompting have been extensively documented in prior work~\citep{sprague2025cot}. Our experiments are therefore not intended as a continuation of those results, but instead as a controlled setup that mirrors our theoretical analysis. In particular, real-world benchmarks make it difficult to disentangle the effect of CoT from other confounding factors, such as prior knowledge embedded in pretrained LLMs. To address this, we first design synthetic tasks that the model has never encountered before, allowing us to isolate and probe the theoretical factors of interest.
\subsection{Synthetic experiment setup}
Our synthetic benchmark is designed to cleanly manipulate the two factors singled out by the theory, \emph{transition alignment} and \emph{intermediate-step noise}, while holding everything else fixed. Each instance is a two-step process. We sample a triple $(x_0,r_1,r_2)$ where $x_0$ is an integer state and each $r_t$ indexes a small stochastic \emph{local rule} acting on the current state (e.g., with probability $0.8$ apply $+1$, otherwise apply $-2$). The \textbf{same} condition enforces the rule of $r_1$ aligns with $r_2$, whereas $r_1$ and $r_2$ represents different local rules in \textbf{diff} condition. For each condition we compare \textbf{CoT} against \textbf{NonCoT}, where we give a specific LLM $n$ context samples generated by the stochastic calculation and expect the model to compute the hard-max result for the query. Note that in this task, we design the random local rule to satisfy the consistency in Assumption \ref{assum:consistent max}, enforcing the hard-max selected by the sequential local rule or the global rule are the same. We sweep $n\in\{8,\ldots,60\}$ and aggregate results over multiple random seeds.

There are several reasons why we select this setup: (i) Enforcing \textbf{same} vs.\ \textbf{diff} isolates the notion of \emph{transition alignment}. (ii) The rule family is intentionally simple so that the ground-truth hard-max at each step and the composed end-to-end hard-max are unambiguous. This lets us measure accuracy without confounding from parsing or generation. (iii) Noise is controlled via the rule probabilities, allowing us to increase the chance that the local hard-max is flipped while leaving everything else unchanged. This separation mirrors the theoretical quantities $\Delta_P$ and $\Delta_Q$. These advantages avoid spurious gains from unrelated factors in complicated experiments like \citep{sprague2025cot,wei2022chain,wang2023self,yao2023tree}.

\textbf{Metrics.} Accuracy is the fraction of test queries whose final prediction equals the end-to-end hard-max index (which, by construction, agrees with composing the per-step hard-max indices). To visualize sample complexity, for a target accuracy threshold $\tau$ we compute the least $n$ that reaches $\tau$ for CoT and NonCoT and plot $\Delta n(\tau)=n_{\text{CoT}}(\tau)-n_{\text{NonCoT}}(\tau)$. Negative values mean CoT achieves the target with fewer samples.

\subsection{Synthetic results and analysis}
\paragraph{Alignment vs.\ misalignment.} From Fig.~\ref{fig:synthetic_same_diff}, CoT helps in both \textbf{same} and \textbf{diff}, but the effect size depends strongly on alignment. With aligned kernels (\textbf{same}), CoT consistently dominates NonCoT over the full range of $n$, and the gap widens at higher accuracy thresholds. With misaligned kernels (\textbf{diff}), improvements are present but uniformly smaller, and can flatten at stricter thresholds. This pattern matches the theory: when transitions align, each trajectory yields multiple informative votes for the \emph{same} kernel (here $T{=}2$). When transitions differ, votes are split across kernels and cannot deliver the same $1/T$ efficiency. Note that for misaligned kernel case, there are $n$ such that CoT is not better than direct inference, which also matches our analysis.

\textbf{Noise sensitivity.} By Fig.~\ref{fig:synthetic_noise}, increasing intermediate-step noise leads to a larger \emph{relative} advantage for CoT. The reason is structural: higher noise reduces both the local margin $\Delta_P$ and the global margin $\Delta_Q$, but due to compounding across steps, $\Delta_Q$ shrinks faster. CoT is therefore more robust to noise than estimating the end-to-end hard-max from final outcomes alone. This is consistent with empirical observations that noisy or unfaithful chains can change outcomes markedly and that aggregation over trajectories helps \citep{prabhakar2024deciphering,wang2023self}.



\begin{wrapfigure}{r}{0.52\textwidth}
  \centering
  \includegraphics[width=\linewidth]{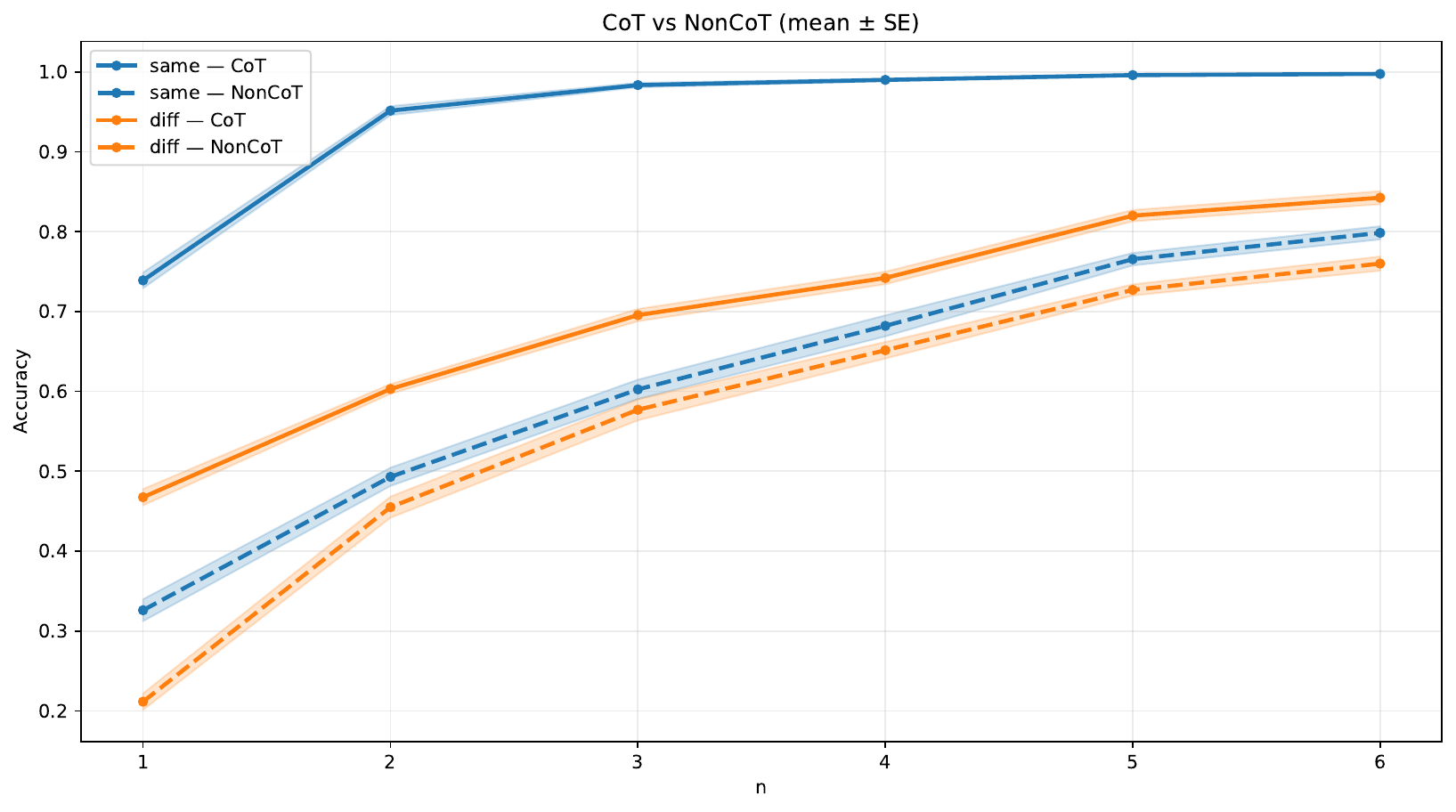}
  \caption{\textbf{Multi-step modular addition (practical check).} Accuracy vs.\ $n$ for CoT and NonCoT when both steps add the \textbf{same} number versus \textbf{different} numbers (fixed modulus). The larger CoT gain in the aligned (\textbf{same}) case corroborates the synthetic findings in a more realistic arithmetic setting. Error bands show mean~$\pm$~SE.}
  \label{fig:mod_add}
\end{wrapfigure}
\subsection{Realistic experiment: modular addition}
While the synthetic data closely mirrors our theoretical setup, it lacks practical relevance. To bridge theory and practice, we therefore turn to a simple arithmetic problem, multi-step modular addition, that has clear real-world value yet remains structured enough for controlled analysis. Specifically, the task is to calculate long additions modular a certain number. To avoid conflating our discussion with heterogeneous step difficulties, we design a simplified version where each step involves an addition of comparable form, differing only in whether all steps add the \emph{same} or \emph{different} numbers.

We compare CoT and NonCoT when all steps add the same number versus different numbers. The result in Fig.~\ref{fig:mod_add} echoes the synthetic study: CoT delivers a substantially larger boost in the \textbf{same} condition, while improvements are smaller in \textbf{diff}. This task is less ``clean" (it involves natural-language templates and minor formatting variations), but it is more practical and still isolates the alignment factor: when both steps share the same local operation, trajectories provide reusable per-step evidence toward the same kernel; on the contrary, the benefit of per-step counting shrinks. Together with the synthetic results, this supports the claim that \emph{transition alignment} is a decisive driver of CoT’s inference-time sample efficiency.
\subsection{Realistic experiment: City--State rankings}
\label{sec:city_state}

We test transition alignment on a corpus-based task built from U.S. city/state rankings under two criteria:
population and area. We collected the data and built a two-hop QA dataset. Each instance in the dataset directly asks for \emph{``the $X$-th largest city in the $Y$-th largest state''},
where ``largest'' is instantiated by a chosen criterion for the state rank and a chosen criterion for the city-within-state rank.
We treat the criterion (population vs.\ area) as the step ``skill'': \textbf{same-skill} uses the same criterion for both steps
(pop--pop or area--area), while \textbf{diff-skill} mixes criteria (pop--area or area--pop).

We sweep $n\in\{2,4,6,8,10\}$ in-context demonstrations and compare CoT vs.\ NonCoT (full details in Appendix~\ref{app:city_state_details}).
Fig.~\ref{fig:city_state} shows that CoT improves over NonCoT in both settings, with a consistently larger gain in the aligned
same-skill condition.

\begin{figure}[t]
  \centering
  \includegraphics[width=\linewidth]{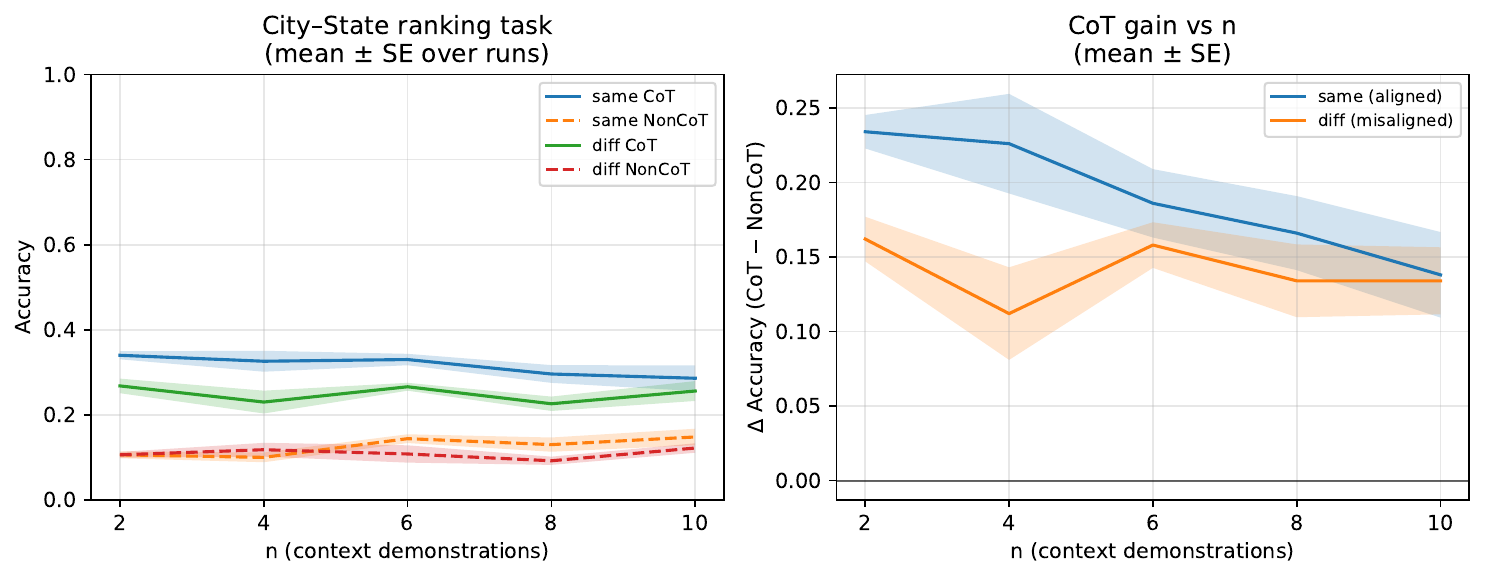}
  \caption{\textbf{City--State rankings.} Accuracy vs.\ $n$ (left) and CoT gain $\Delta\text{Acc}$ (right) under
  \textbf{same-skill} vs.\ \textbf{diff-skill}. Error bands: mean~$\pm$~SE.}
  \label{fig:city_state}
\end{figure}

\section{Conclusion and Discussions}
We provided a mechanism-level account of when CoT improves inference-time sample efficiency by modeling stepwise reasoning as a finite-state Markov chain, motivated by the autoregressive nature of CoT and the finite context windows in practice.
Two structural properties govern the performance of CoT: 
\begin{enumerate*}[label=(\roman*)]
    \item \textbf{transition alignment} that characterizes whether all steps share the same local transition, and
    \item \textbf{decision margins/noise} that describes the per-step confidence.
\end{enumerate*}
We rigorously establish the decisive effects of both properties. In particular, we show that 
\begin{enumerate*}[label=(\roman*)]
    \item CoT yields a $1/T$ gain over direct inference only when the local transitions in successive steps are aligned; while
    \item the decision margins control the gain of CoT, which improves when per-step evidence is comparatively more robust than end-to-end decisions.
\end{enumerate*}
These theoretical findings are corroborated by controlled experiments on synthetic and modular addition tasks.

\paragraph{Implications for implicit thinking.} Our Markovian formulation treats multi-step reasoning as a trajectory of latent states with step-wise transitions. An explicit CoT transcript is one of the intuitive ways to expose this trajectory. 
Whether the intermediate reasoning can be reused across steps, the efficiency gain of CoT in our analysis comes from the underlying dynamics, instead of expressing the intermediate states as human-readable text. 
This perspective casts light on implicit thinking: one can keep multi-step state evolution, but change the interface that reveals intermediate states (e.g., compress them or not reveal them at all). Empirically, this suggests comparing implicit and explicit interfaces under matched test-time compute. It is also helpful to control for the same structural descriptors in our theory (e.g., alignment and intermediate noise), so we can separate improvements in internal dynamics from improvements due to the readout format.

\paragraph{Future directions.}
We see two complementary future avenues. From the explanatory perspective, existing works on CoT theory mostly relies on simplified abstractions. A next step is to push toward more practical models that still admit analysis, e.g., higher-order or semi-Markov structure, hierarchical/graph decompositions, or continuous latent states. Then the factors can be characterized across a wider range of real reasoning pipelines.
From the application perspective, we seek to extend the measurement of the key structural properties identified by the theory to complex real-world tasks, which can potentially steer the design of more effective, aligned, and robust contexts for CoT prompting.

\newpage

\bibliography{bib}


\newpage
\appendix

\section{Preliminaries and Core Lemma}
\label{app:prelims}

\paragraph{Notation.}
We write $[k]=\{1,\dots,k\}$ for the finite state space.
Vectors are row vectors unless stated otherwise.
For a probability vector $p\in\Delta^{k-1}$, let $p_{(1)}\ge p_{(2)}\ge\cdots$ denote the order statistics of its coordinates and define the \emph{(row) margin}
\[
\mathrm{margin}(p):=p_{(1)}-p_{(2)}.
\]
For a row-stochastic matrix $M\in\mathbb{R}^{k\times k}$, we define
\(
\mathrm{margin}(M):=\min_{i\in[k]}\mathrm{margin}(M_i).
\)

\paragraph{Divergences and distances.}
For probability laws $\nu,\pi$ on $[k]$, the $\chi^2$-divergence is
\[
\chi^2(\nu\|\pi):=\sum_{i=1}^k \frac{(\nu_i-\pi_i)^2}{\pi_i}
=\Bigl\|\frac{\nu}{\pi}-\mathbf{1}\Bigr\|_{L^2(\pi)}^2,
\quad \text{where } \frac{\nu}{\pi}(i):=\nu_i/\pi_i.
\]
It controls total variation (TV) via the inequality
\(
\|\nu-\pi\|_{\mathrm{TV}}\le \tfrac12\sqrt{\chi^2(\nu\|\pi)}.
\)

\paragraph{Pseudo-spectral gap.}
Let $P$ be a (possibly nonreversible) Markov kernel on $[k]$ with stationary distribution $\pi$, and let $P^*$ denote its time-reversal on $L^2(\pi)$:
$
P^*(x,dy)=\frac{\pi(dy)\,P(y,dx)}{\pi(dx)}.
$
Following \citet[Eq.~(3.3)]{paulin2015psg}, the \emph{pseudo-spectral gap} of $P$ is
\begin{equation}
\label{eq:psg-def}
\gamma_{\mathrm{ps}}
:=
\max_{k\ge 1}\, \frac{\gamma\bigl((P^*)^k P^k\bigr)}{k},
\end{equation}
where $\gamma(A)$ denotes the (absolute) spectral gap of the self-adjoint operator $A$ on $L^2(\pi)$.
For reversible chains, $\gamma_{\mathrm{ps}}$ coincides with the usual \(L^2\) spectral gap up to a constant factor, while for nonreversible chains it generalizes multiplicative reversibilization \citep[][Remark~3.2]{paulin2015psg}. Moreover, $\gamma_{\mathrm{ps}}$ relates to mixing time via \citet[Prop.~3.4]{paulin2015psg}.

\paragraph{Context trajectories and counts.}
We observe $n$ i.i.d.\ context trajectories (inference-time only).
For the homogeneous case (all steps use $P$), write $X^{(\ell)}_{0:T}$ for the $\ell$-th path and
\(
N_i^{(t)}:=\sum_{\ell=1}^n \mathbf{1}\{X^{(\ell)}_t=i\}
\)
for the number of visits to state $i$ at step $t$.
We also use the pooled count
\(
N_i:=\sum_{t=0}^{T-1} N_i^{(t)}.
\)
For direct inference and the heterogeneous case, we further write
\(
N_i^{(0)}:=\sum_{\ell=1}^n \mathbf{1}\{X^{(\ell)}_0=i\}.
\)

Now we introduce the main lemma used in the proof.
\begin{lem}[Multinomial top-1 identification (restatement of Lemma~\ref{lem:multinomial-top1})]
\label{lem:top1-app}
Let $p\in\Delta^{k-1}$ with margin $\Delta_p:=p_{(1)}-p_{(2)}>0$, and let $\hat p$ be the empirical frequencies from $N$ i.i.d.\ draws.
There exists an absolute constant $C>0$ such that if
\[
N \ge \frac{C}{\Delta_p^2}\,\log\!\frac{k}{\delta},
\]
then $\arg\max_j \hat p_j=\arg\max_j p_j$ with probability at least $1-\delta$.
\end{lem}

\begin{proof}
Let $j^\star=\arg\max_j p_j$ and for each $j\neq j^\star$ define $Z_\ell:=\mathbf{1}\{X_\ell=j^\star\}-\mathbf{1}\{X_\ell=j\}\in[-1,1]$ with $\mathbb{E}Z_\ell=p_{j^\star}-p_j\ge \Delta_p$.
Then $\hat p_{j^\star}-\hat p_j=\frac{1}{N}\sum_{\ell=1}^N Z_\ell$ and by Hoeffding
\(
\Pr\{\hat p_{j^\star}\le \hat p_j\}\le \exp(-2N\Delta_p^2).
\)
Union-bounding over $j\neq j^\star$ yields
\(
\Pr\{\exists j\neq j^\star:\hat p_j\ge \hat p_{j^\star}\}\le (k-1)\exp(-2N\Delta_p^2)\le \delta
\)
under the stated $N$ (for a suitable absolute $C$).
\end{proof}

\section{Proofs of the Main Theorems}
\label{app:proofs}

We restate each theorem and then give a full proof. All three proofs invoke Lemma~\ref{lem:top1-app} and the homogeneous case additionally uses a pooled-coverage lemma whose proof relies on two results from \citet{paulin2015psg}.

\subsection{Proof of Theorem~\ref{thm:direct} (Direct inference)}

\begin{thm}[Restatement of Theorem~\ref{thm:direct}]
\label{app:thm:direct}
Under Assumptions \ref{assum:init dist} and \ref{assum:Q margin}, if the number of samples $n$ satisfies
\[
n= \Theta\left(\frac{\log(k/\delta)}{\mu_{\min}\Delta_Q^2}\right),
\]
then with probability at least $1-\delta$, we have $\hat{j}^*(i)=\arg\max_j Q_{ij}$ for all $i$ with direct inference.
\end{thm}

\begin{proof}
Let $N=n$ and $N_i^{(0)}$ be the number of contexts with $x_0=i$.
Since $\mu_{\min}>0$ (Assumption~\ref{assum:init dist}), a Chernoff bound gives
\[
\Pr\!\big\{N_i^{(0)}\le \tfrac12 N\mu_i\big\}\le \exp(-N\mu_i/8)\le \exp(-N\mu_{\min}/8).
\]
Union-bounding over $i\in[k]$, we obtain that if
\(
N\ge C_0\,\mu_{\min}^{-1}\log(k/\delta)
\)
(for an absolute $C_0$), then with probability at least $1-\delta/2$,
\begin{equation}
\label{eq:dir-coverage}
\min_{i\in[k]} N_i^{(0)} \ge \tfrac12 N\mu_{\min}.
\end{equation}
Conditioned on $x_0=i$, the $N_i^{(0)}$ terminal labels $x_T$ are i.i.d.\ from the multinomial distribution $Q_i$.
By Assumption~\ref{assum:Q margin}, the row margin obeys $\mathrm{margin}(Q_i)\ge \Delta_Q>0$ for all $i$.
Applying Lemma~\ref{lem:top1-app} with $N\gets N_i^{(0)}$ and $\Delta_p\gets \Delta_Q$ yields
\[
\Pr\!\left\{\hat j^*(i)\neq\arg\max_j Q_{ij}\;\middle|\;N_i^{(0)}\right\}
\le k\cdot \exp\!\left(-C\,\Delta_Q^2\,N_i^{(0)}\right),
\]
for an absolute $C>0$.
On the event \eqref{eq:dir-coverage}, the RHS is at most $k\exp(-C\,\Delta_Q^2\,N\mu_{\min}/2)$.
Another union bound over $i\in[k]$ shows that all $k$ rows are identified correctly with probability at least $1-\delta/2$ provided
\(
N\gtrsim \frac{1}{\mu_{\min}\Delta_Q^2}\log(k/\delta).
\)
Combining with \eqref{eq:dir-coverage} via a union bound establishes the theorem with the stated $\Theta(\cdot)$ rate.
\end{proof}

\subsection{Proof of Theorem~\ref{thm:same} (Homogeneous CoT)}
\label{app:homog-tools-and-proofs}

We introduce two technical lemmas for Markov chain based on pseudo spectral gap.

\begin{lem}[TV decay via pseudo-spectral gap {\citep[Prop.~3.4]{paulin2015psg}}]
\label{lem:paulin-tv}
Let $P$ be a (possibly nonreversible) Markov kernel on $[k]$ with stationary distribution $\pi$.
Let $P^*$ denote its time-reversal on $L^2(\pi)$, and define the pseudo-spectral gap
\begin{equation}
\label{eq:psg-def-b2}
\gamma_{\mathrm{ps}}
:=
\max_{m\ge 1}\, \frac{\gamma\!\big((P^*)^{m} P^{m}\big)}{m},
\end{equation}
where $\gamma(A)$ is the (absolute) spectral gap of the self-adjoint operator $A$ on $L^2(\pi)$.
For any initial law $q$ absolutely continuous w.r.t.\ $\pi$, write
$N_q := \mathbb{E}_{\pi}\!\big[\big(\tfrac{dq}{d\pi}\big)^2\big]=1+\chi^2(q\|\pi)$.
Then, for all $n\ge 1$,
\begin{equation}
\label{eq:paulin-tv-b2}
d_{\mathrm{TV}}(q P^{n}, \pi)
\le
\frac{1}{2}\,
(1-\gamma_{\mathrm{ps}})^{(n-1/\gamma_{\mathrm{ps}})/2}\,\sqrt{N_q-1}.
\end{equation}
In particular, letting $\chi_0:=\sqrt{\chi^2(q\|\pi)}=\sqrt{N_q-1}$ and summing the geometric decay,
\begin{equation}
\label{eq:avg-tv-psg-b2}
\frac{1}{T}\sum_{t=0}^{T-1} d_{\mathrm{TV}}(q P^{t},\pi)
\le
\frac{C}{T}\cdot \frac{\chi_0}{\gamma_{\mathrm{ps}}},
\end{equation}
for a universal constant $C>0$.
\end{lem}

\begin{lem}[Bernstein inequality for additive functionals {\citep[Thm.~3.11]{paulin2015psg}}]
\label{lem:paulin-bernstein-raw}
Let $(X_t)_{t\ge 0}$ be a (possibly nonreversible) Markov chain with stationary law $\pi$ and pseudo-spectral gap $\gamma_{\mathrm{ps}}>0$ as in \eqref{eq:psg-def-b2}, and assume $X_0\sim\pi$ (stationary start).
Let $f:[k]\to\mathbb{R}$ be bounded with $\|f-\mathbb{E}_\pi f\|_\infty \le b$ and variance $\mathrm{Var}_\pi(f)=\sigma^2$.
Then there exist absolute constants $a_1,a_2>0$ such that for all $u>0$ and $T\in\mathbb{N}$,
\begin{equation}
\label{eq:paulin-bernstein-raw}
\Pr\!\Bigg\{\Bigg|\sum_{t=0}^{T-1}\big(f(X_t)-\mathbb{E}_\pi f\big)\Bigg| \ge u\Bigg\}
\le
2\exp\!\left(
-\,
\frac{a_1\,\gamma_{\mathrm{ps}}\,u^2}{\,T\,\sigma^2+a_2\,b\,u\,}
\right).
\end{equation}
\end{lem}

With the lemmas above, we can prove the following important lemma that gives a lower bound on the total coverage along the Markov chain.
\begin{lem}[Aggregated step coverage under mixing (restatement of Lemma~\ref{lem:min-count-markov})]
\label{lem:cover-b2}
Assume $P^{(1)}=\cdots=P^{(T)}=:P$ with stationary law $\pi$ ($\pi_{\min}>0$), pseudo-spectral gap $\gamma_{\mathrm{ps}}>0$, and $\chi_0=\sqrt{\chi^2(\mu\|\pi)}$.
Let $N_i^{(t)}$ be the number of visits to state $i$ at step $t$ across $n$ i.i.d.\ trajectories, and $N_i:=\sum_{t=0}^{T-1} N_i^{(t)}$.
There exist absolute $c_1,c_2>0$ such that, for any $\delta\in(0,1)$, with probability at least $1-\delta$,
\[
\min_{i\in[k]} N_i
\ge
nT\,\pi_{\min}
-
\frac{n\,\chi_0}{\gamma_{\mathrm{ps}}}
-
\sqrt{\frac{c_1\,nT}{\gamma_{\mathrm{ps}}}\,\log\!\frac{k}{\delta}}
- c_2.
\]
\end{lem}

\begin{proof}[Proof of Lemma~\ref{lem:cover-b2}]
Fix $i\in[k]$ and set $f_i(x)=\mathbf{1}\{x=i\}\in[0,1]$.
For one trajectory $X_{0:T-1}$, write
\(
Y_i:=\sum_{t=0}^{T-1} f_i(X_t).
\)

\emph{(Bias.)} By Lemma~\ref{lem:paulin-tv}, in particular \eqref{eq:avg-tv-psg-b2},
\[
\mathbb{E}Y_i
=\sum_{t=0}^{T-1} (\mu P^t)_i
\;\ge\;
T\pi_i - \frac{C\,\chi_0}{\gamma_{\mathrm{ps}}}.
\]
\emph{(Fluctuation.)} Consider the centered sum around $\mathbb{E}_\pi f_i$:
\[
S_T^\circ := \sum_{t=0}^{T-1} \big(f_i(X_t)-\mathbb{E}_\pi f_i\big).
\]
Here $\sigma^2=\mathrm{Var}_\pi(f_i)\le \tfrac14$ and $b=\|f_i-\mathbb{E}_\pi f_i\|_\infty\le 1$.
Applying Lemma~\ref{lem:paulin-bernstein-raw} with these $(\sigma^2,b)$ yields, for absolute $c_1'>0$,
\begin{equation}
\label{eq:bernstein-simplified}
\Pr\{|S_T^\circ|\ge u\}
\le
2\exp\!\left(
-\,c_1'\,\frac{\gamma_{\mathrm{ps}}\,u^2}{T+u}
\right).
\end{equation}
The centered sum we actually need is
\(
S_T:=\sum_{t=0}^{T-1}\big(f_i(X_t)-\mathbb{E}f_i(X_t)\big)
= S_T^\circ - \sum_{t=0}^{T-1}\big(\mathbb{E}f_i(X_t)-\mathbb{E}_\pi f_i\big).
\)
The deterministic bias term is bounded by total variation:
\[
\big|\sum_{t=0}^{T-1}(\mathbb{E}f_i(X_t)-\mathbb{E}_\pi f_i)\big|
\le \sum_{t=0}^{T-1} \|\mathcal{L}(X_t)-\pi\|_{\mathrm{TV}}
\le C\,\chi_0/\gamma_{\mathrm{ps}}
\]
by Lemma~\ref{lem:paulin-tv}.
Thus $S_T$ inherits the tail in \eqref{eq:bernstein-simplified} up to a shift by a deterministic constant. Absorbing that shift into constants, we can write
\begin{equation}
\label{eq:bernstein-final}
\Pr\{|S_T|\ge u\}
\le
2\exp\!\left(
-\,c\,\frac{\gamma_{\mathrm{ps}}\,u^2}{T+u}
\right)
\qquad (c>0\ \text{absolute}).
\end{equation}

\emph{(From one trajectory to $n$ trajectories.)}
Let $Y_i^{(1)},\dots,Y_i^{(n)}$ be i.i.d.\ copies, and $Z_\ell:=Y_i^{(\ell)}-\mathbb{E}Y_i$.
By \eqref{eq:bernstein-final}, each $Z_\ell$ is sub-exponential with variance proxy $O(T/\gamma_{\mathrm{ps}})$ and scale proxy $O(1/\gamma_{\mathrm{ps}})$.
A Bernstein inequality for independent sub-exponential sums gives absolute $C_1,C_2>0$ such that, with probability at least $1-\delta/k$,
\[
\sum_{\ell=1}^{n} Y_i^{(\ell)}
\;\ge\;
n\,\mathbb{E}Y_i
\;-\;
\sqrt{C_1\,\frac{nT}{\gamma_{\mathrm{ps}}}\,\log\!\frac{k}{\delta}}
\;-\;
C_2\,\frac{1}{\gamma_{\mathrm{ps}}}\,\log\!\frac{k}{\delta}.
\]
A union bound over $i\in[k]$ and the lower bound on $\mathbb{E}Y_i$ complete the proof (absorbing the linear-in-log term into $c_2$).
\end{proof}

\begin{thm}[Restatement of Theorem~\ref{thm:same}]
\label{thm:theory-homog-b2}
Under Assumptions \ref{assum:init dist}, \ref{assum:consistent max} and \ref{assum:transition}, if the number of samples $n$ satisfies
\[
n=\Theta\left(\left(\frac{1}{T\pi_{\min}\Delta_P^2\,r}+\frac{1}{T\pi_{\min}^2\,r^2}\right)\log\!\frac{k}{\delta}\right),
\qquad
r:=1-\frac{\chi_0}{T\pi_{\min}\gamma_{\mathrm{ps}}},
\]
then with probability at least $1-\delta$, we have $\hat{j}^*(i)=\arg\max_{j}Q_{ij}$ for all $i$ with CoT inference under homogeneous transitions.
\end{thm}

\begin{proof}
By Lemma~\ref{lem:cover-b2}, with probability $\ge 1-\delta/2$,
\(
\min_i N_i \ge
nT\,\pi_{\min}
-
\frac{n\,\chi_0}{\gamma_{\mathrm{ps}}}
-
\sqrt{\frac{c_1\,nT}{\gamma_{\mathrm{ps}}}\log\!\frac{k}{\delta}} - c_2.
\)
Assumption~\ref{assum:transition} ensures $\mathrm{margin}(P)\ge \Delta_P>0$.
Applying Lemma~\ref{lem:top1-app} row-wise with a union bound gives the stated sample-size condition (first term from pooled $T$ transitions per trajectory; second from $\sqrt{nT/\gamma_{\mathrm{ps}}}$ fluctuations).
Assumption~\ref{assum:consistent max} then maps local argmaxes to the global argmax of $Q$.
\end{proof}

\subsection{Proof of Theorem~\ref{thm:diff} (Heterogeneous CoT)}

\begin{thm}[Restatement of Theorem~\ref{thm:diff}]
\label{app:thm:hetero}
Under Assumptions \ref{assum:init dist}, \ref{assum:consistent max} and \ref{assum:diff}, if the number of samples $n$ satisfies
\[
n=\Theta\left(\frac{\log(Tk/\delta)}{q_{\min}\Delta^2}\right),
\]
then with probability at least $1-\delta$, we have $\hat{j}^*(i)=\arg\max_{j}Q_{ij}$ for all $i$ with CoT inference with heterogeneous transitions.
\end{thm}

\begin{proof}
At step $t\in\{0,\dots,T-1\}$ the state distribution is $\mu^{(t)}=\mu P^{(1)}\cdots P^{(t)}$.
Let $N_i^{(t)}$ be the number of visits to $i$ at step $t$ across the $n$ trajectories. Then $N_i^{(t)}\sim \mathrm{Bin}(n,\mu_i^{(t)})$ independently across trajectories.
By Chernoff and Assumption~\ref{assum:diff},
\(
\Pr\{N_i^{(t)}\le \tfrac12 n \mu_i^{(t)}\}\le \exp(-n \mu_i^{(t)}/8)\le \exp(-n q_{\min}/8).
\)
Union-bounding over the $Tk$ pairs $(t,i)$ shows that, if $n\ge C_0\,q_{\min}^{-1}\log(Tk/\delta)$, then with probability at least $1-\delta/2$,
\begin{equation}
\label{eq:het-coverage}
\min_{t\in\{0,\dots,T-1\}}\;\min_{i\in[k]} N_i^{(t)} \;\ge\; \tfrac12 n\,q_{\min}.
\end{equation}
Conditioned on $x_t=i$, the next state $x_{t+1}$ is multinomial with row $P^{(t)}_i$.
Assumption~\ref{assum:diff} provides the \emph{uniform} local margin
\(
\min_{t,i}\mathrm{margin}(P^{(t)}_i)\ge \Delta>0.
\)
Applying Lemma~\ref{lem:top1-app} to each pair $(t,i)$ with $N_i^{(t)}$ samples and margin $\Delta$, and union-bounding over all $Tk$ pairs, we conclude that \eqref{eq:het-coverage} together with
\(
n\gtrsim (q_{\min}\Delta^2)^{-1}\log(Tk/\delta)
\)
ensures that all per-step rowwise argmax decisions are correct with probability $\ge 1-\delta/2$.
Finally, by Assumption~\ref{assum:consistent max}, composing these local decisions yields the global argmax of $Q$ for every initial state, completing the proof.
\end{proof}

\section{Experimental Details}

All experiments on math problems use \emph{DeepSeek-Math-7B Base} \citep{deepseekmath2024} with identical decoding settings across conditions and the city-state experiment use \emph{Llama-3-8b-instruct} \citep{dubey2024llama}. We study \emph{inference time} only in the experiments, so no training, fine-tuning, or gradient updates are performed. The model simply consumes the provided context demonstrations and predicts via the count-and-argmax decision rule analyzed in the main text.

\subsection{Synthetic experiments}
We use a two-symbol alphabet $\{A,B\}$ for local rules. Each symbol $r\in\{A,B\}$ denotes a two-point stochastic update on the current discrete state: with probability $p(r)$ apply one update (e.g., $+u(r)$), and with probability $1-p(r)$ apply the other (e.g., $-v(r)$). The states are wrapped to remain in $[k]$. A test instance is a two-step tuple $(x_0,r_1,r_2)$ with $x_0\sim\mu$ (uniform by default) and $r_t\in\{A,B\}$. The key manipulation is how the \emph{semantics} of the letters are tied across steps. In \textbf{same} (aligned), the letter-to-operation mapping is fixed across steps. In \textbf{diff} (misaligned), the step-2 mapping is altered so that the same letter need not mean the same stochastic rule as in step~1. For each configuration we draw $n\in\{8,\dots,60\}$ i.i.d.\ context trajectories. For context samples, \textbf{NonCoT} records only terminals $(x_0,r_1,r_2,x_2)$, while \textbf{CoT} records full paths $(x_0,r_1,x_1,r_2,x_2)$. We report accuracy ($\pm$ SE over seeds) and the proxy $\Delta n(\tau)=n_{\text{CoT}}(\tau)-n_{\text{NonCoT}}(\tau)$. Noise is controlled by moving $p(A)$ and $p(B)$ toward $1/2$ (smaller per-step margin) while holding all other factors fixed.

\subsection{Realistic experiment: modular addition}
Each instance specifies modulus $M$, initial $x_0\in\{0,\dots,M-1\}$, and addends $(a_1,\dots,a_L)$. The target is $x_L\equiv x_0+a_1+\cdots+a_L\ (\mathrm{mod}\ M)$. We compare \textbf{same} ($a_1\equiv\cdots\equiv a_L$) vs.\ \textbf{diff} ($a_1,\dots, a_L$ drawn independently). Contexts contain either only terminals $(x_L)$ (NonCoT) or also $x_l\equiv x_0+a_1+\cdots+a_l$ for all $l$ (CoT), with the same budget $n$ and seeds as in the synthetic suite. Estimation and reporting metrics follow identically from the synthetic protocol, providing a practical check that alignment drives CoT gains beyond the idealized setting.

\subsection{Realistic experiment: city--state rankings}
\label{app:city_state_details}
We build a text corpus of U.S.\ city/state rankings under two criteria: \textbf{population} and \textbf{area}.
Facts are rendered in a fixed template for (i) states (global ranks) and (ii) cities within each state (within-state ranks),
for both criteria. Each query is phrased as \emph{``the $X$-th largest city in the $Y$-th largest state''} with explicit
criteria for the state step and the city-within-state step. Ground truth is obtained by (1) selecting the $Y$-th state
under the chosen state criterion, then (2) selecting the $X$-th city within that state under the chosen city criterion,
and outputting the resulting city name. We define \textbf{same-skill} by using the same criterion for both steps
(pop--pop or area--area) and \textbf{diff-skill} by mixing criteria (pop--area or area--pop). All prompts include the corpus and $n$ in-context demonstrations. CoT demonstrations expose the two intermediate steps
(state identification, then city identification) before the final answer. NonCoT demonstrations output only the final city name.
Evaluation in both cases is on the final answer only.

\end{document}

%% file: Command.tex
\newcommand{\beq}{\vspace{0mm}\begin{equation}}
\newcommand{\eeq}{\vspace{0mm}\end{equation}}
\newcommand{\beqs}{\vspace{0mm}\begin{eqnarray}}
\newcommand{\eeqs}{\vspace{0mm}\end{eqnarray}}
\newcommand{\barr}{\begin{array}}
\newcommand{\earr}{\end{array}}

\newtheorem{thm}{Theorem} [section]

\newtheorem{lem}[thm]{Lemma}

\newtheorem{assum}{Assumption} [section]

\newtheorem{remark}[thm]{Remark}

\usepackage{booktabs}
\usepackage{array}
\newcolumntype{L}[1]{>{\raggedright\let\newline\\\arraybackslash\hspace{0pt}}m{#1}}
\newcolumntype{C}[1]{>{\centering\let\newline\\\arraybackslash\hspace{0pt}}m{#1}}
\newcolumntype{R}[1]{>{\raggedleft\let\newline\\\arraybackslash\hspace{0pt}}m{#1}}
